\documentclass[runningheads]{llncs}
\usepackage{graphicx}
\usepackage{amssymb}
\usepackage{booktabs}
\usepackage{amsmath}
\usepackage{tabularx}
\usepackage{caption}
\usepackage{subcaption}
\usepackage{multirow}
\usepackage[breaklinks=true,bookmarks=false]{hyperref}

\usepackage{enumitem}
\usepackage{mathtools}
\newtheorem{assumption}{Assumption}

\usepackage{makecell}
\usepackage{algorithm,algorithmic}

\DeclareMathOperator*{\argmin}{arg\,min}
\newcommand{\etal}{\textit{et al}.}
\begin{document}
\title{A Proximal Algorithm for Network Slimming\thanks{The work was partially supported by NSF grants DMS-1854434, DMS-1952644, DMS-2151235, and a Qualcomm Faculty Award. }}
%
%
\author{Kevin Bui\inst{1} \and
Fanghui Xue\inst{1} \and
Fredrick Park\inst{2}\and Yingyong Qi\inst{1} \and Jack Xin\inst{1}}
\authorrunning{Bui et al.}
%
\institute{University of California, Irvine, CA 92697, USA\\
\email{\{kevinb3, fanghuix, yqi, jack.xin\}}@uci.edu\and
Whittier College, Whittier, CA, 90602, USA \\
\email{fpark1}@whittier.edu}
\maketitle              
\begin{abstract}
As a popular channel pruning method for convolutional neural networks (CNNs), network slimming (NS) has a three-stage process: (1) it trains a CNN with $\ell_1$ regularization applied to the scaling factors of the batch normalization layers; (2) it removes channels whose scaling factors are below a chosen threshold; and (3) it retrains the pruned model to recover the original accuracy. This time-consuming, three-step process is a result of using subgradient descent to train CNNs. Because subgradient descent does not exactly train CNNs towards  sparse, accurate structures, the latter two steps are necessary. Moreover, subgradient descent does not have any convergence guarantee. Therefore, we develop an alternative algorithm called proximal NS. Our proposed algorithm trains CNNs towards sparse, accurate structures, so identifying a scaling factor threshold is unnecessary and fine tuning the pruned CNNs is optional. Using Kurdyka-Łojasiewicz assumptions, we establish global convergence of proximal NS. Lastly, we validate the efficacy of the proposed algorithm on VGGNet, DenseNet and ResNet on CIFAR 10/100. Our experiments demonstrate that after one round of training, proximal NS yields a CNN with competitive accuracy and compression. 
\end{abstract}

\section{Introduction}
In the past decade, convolutional neural networks (CNNs) have revolutionized computer vision in various applications, such as image classification \cite{he2016deep,simonyan2014very,zagoruyko2016wide} and object detection \cite{girshick2014rich,huang2017speed,ren2015faster}. CNNs are able to internally generate diverse, various features through its multiple hidden layers, totaling millions of weight parameters to train and billions of floating point operations (FLOPs) to execute. Consequently, highly accurate CNNs are impractical to store and implement on resource-constrained devices, such as mobile smartphones. 

To compress CNNs into lightweight models, several directions, including weight pruning \cite{aghasi2017net,han2015learning}, have been investigated. Channel pruning \cite{liu2017learning,wen2016learning} is currently a popular direction because it can significantly reduce the number of weights needed in a CNN by removing any redundant channels. One straightforward approach to channel pruning is network slimming (NS) \cite{liu2017learning}, which appends an $\ell_1$ norm on the scaling factors of the batch normalization layers to the loss function being optimized. Being a sparse regularizer, the $\ell_1$ norm pushes the scaling factors corresponding to the channels towards zeroes. The original optimization algorithm used for NS is subgradient descent \cite{shor2012minimization}, but it has theoretical and practical issues. Subgradient descent does not necessarily decrease the loss function value after each iteration, even when performed exactly with full batch of data \cite{beck2017first}. Moreover, unless with some additional modifications, such as backtracking line search, subgradient descent may not converge to a critical point \cite{noll2014convergence}. When implemented in practice, barely any of the scaling factors have values exactly at zeroes by the end of training, resulting in two issues. First, a threshold value needs to be determined in order to remove channels whose scaling factors are below it. Second, pruning channels with nonzero scaling factors can deteriorate the CNNs' accuracy since these channels are still relevant to the CNN computation. As a result, the pruned CNN needs to be retrained to recover its original accuracy. Therefore, as a suboptimal algorithm, subgradient descent leads to a time-consuming, three-step process.

In this paper, we design an alternative optimization algorithm based on proximal alternating linearized minimization (PALM) \cite{bolte2014proximal} for NS. The algorithm has more theoretical and practical advantages than subgradient descent. Under certain conditions, the proposed algorithm does converge to a critical point. When used in practice, the proposed algorithm enforces the scaling factors of insignificant channels to be exactly zero by the end of training. Hence, there is no need to set a scaling factor threshold to identify which channels to remove. Because the proposed algorithm trains a model towards a truly sparse structure, the model accuracy is preserved after the insignificant channels are pruned, so fine tuning is unnecessary. The only trade-off of the proposed algorithm is a slight decrease in accuracy compared to the original baseline model. Overall, the new algorithm reduces the original three-step process of NS to only one round of training with fine tuning as an optional step, thereby saving the time and hassle of obtaining a compressed, accurate CNN.

\section{Related Works}
Early pruning methods focus on removing redundant weight parameters in CNNs. Han \etal \cite{han2015learning} proposed to remove weights if their magnitudes are below a certain threshold. Aghasi \etal \cite{aghasi2020fast} formulated a convex optimization problem to determine which weight parameters to retain while preserving model accuracy. Creating irregular sparsity patterns, weight pruning is not implementation friendly since it requires special software and hardware to accelerate inference \cite{Li_2020_CVPR,zhao2019variational}.

An alternative to weight pruning is pruning group-wise structures in CNNs. Many works \cite{alvarez2016learning,bui2021structured,li2016pruning,meng2020pruning,scardapane2017group,wen2016learning} have imposed group regularization onto various CNN structures, such as filters and channels. Li \etal \cite{Li_2020_CVPR} incorporated a sparsity-inducing matrix corresponding to each feature map and imposed row-wise and column-wise group regularization onto this matrix to determine which filters to remove. Lin \etal \cite{lin2020hrank} pruned filters that generate low-rank feature maps. Hu \etal \cite{hu2016network} devised network trimming that iteratively removes zero-activation neurons from the CNN and retrains the compressed CNN. Rather than regularizing the weight parameters, Liu \etal \cite{liu2017learning} developed NS, where they applied $\ell_1$ regularization on the scaling factors in the batch normalization layers in a CNN to determine which of their corresponding channels are redundant to remove and then they retrained the pruned CNN to restore its accuracy. Bui \etal \cite{bui2020nonconvex,bui2021improving} investigated nonconvex regularizers as alternatives to the $\ell_1$ regularizer for NS. On the other hand, Zhao \etal \cite{zhao2019variational} applied probabilistic learning onto the scaling factors to identify which redundant channels to prune with minimal accuracy loss, making retraining unnecessary. Lin \etal \cite{lin2019towards} introduced an external soft mask as a set of parameters corresponding to the CNN structures (e.g., filters and channels) and regularized the mask by adversarial learning.

\section{Proposed Algorithm}
In this section, we develop a novel PALM algorithm \cite{bolte2014proximal} for NS that consists of two straightforward, general steps per epoch: stochastic gradient descent on the weight parameters, including the scaling factors of the batch normalization layers, and soft thresholding on the scaling factors.
\subsection{Batch Normalization Layer}
Most modern CNNs have batch normalization (BN) layers \cite{ioffe2015batch} because these layers speed up their convergence and improve their generalization \cite{santurkar2018does}. These benefits are due to normalizing the output feature maps of the preceding convolutional layers using mini-batch statistics. Let $z \in \mathbb{R}^{B \times C \times H \times W}$ denote an output feature map, where $B$ is the mini-batch size, $C$ is the number of channels, and $H$ and $W$ are the height and width of the feature map, respectively. For each channel $i=1, \ldots, C$, the output of a BN layer on each channel $z_i$ is given by
\begin{align}
    z_i' = \gamma_i \frac{z_i-\mu_B}{\sqrt{\sigma^2_B + \epsilon}} + \beta_i,
\end{align}
where $\mu_B$ and $\sigma_B$ are the mean and standard deviation of the inputs across the mini-batch $B$, $\epsilon$ is a small constant for numerical stability, and $\gamma_i$ and $\beta_i$ are trainable weight parameters that help restore the representative power of the input $z_{i}$. The weight parameter $\gamma_i$ is defined to be the scaling factor of channel $i$. The scaling factor $\gamma_i$ determines how important channel $i$ is to the CNN computation as it is multiplied to all pixels of the same channel $i$ within the feature map $z$.
\subsection{Numerical Optimization}
Let $\{(x_i,y_i)\}_{i=1}^N$ be a given dataset, where each $x_i$ is a training input and $y_i$ is its corresponding label or value. Using the dataset $\{(x_i,y_i)\}_{i=1}^N$, we train a CNN with $c$ total channels, where each of their convolutional layers is followed by a BN layer. Let $\gamma \in \mathbb{R}^c$ be the vector of trainable scaling factors of the CNN, where for $i=1,\ldots, c$, each entry $\gamma_i$ is a scaling factor of channel $i$. Moreover, let $W \in \mathbb{R}^n$ be a vector of all $n$ trainable weight parameters, excluding the scaling factors, in the CNN. NS \cite{liu2017learning} minimizes the following objective function:
\begin{align}\label{eq:network_slim}
    \min_{W, \gamma} \frac{1}{N} \sum_{i=1}^N \mathcal{L}(h(x_i,W, \gamma), y_i) + \lambda \|\gamma\|_1,
\end{align}
where $h(x_i, W, \gamma)$  is the output of the CNN predicted on the data point $x_i$;\\ $\mathcal{L}(h(x_i,W, \gamma), y_i)$ is the loss function between the prediction $h(x_i, W, \gamma)$  and ground truth $y_i$, such as the cross-entropy loss function;  and $\lambda > 0$ is the regularization parameter for the $\ell_1$ penalty on the scaling factor vector $\gamma$. In \cite{liu2017learning}, \eqref{eq:network_slim} is solved by a gradient descent scheme with step size $\delta^t$ for each epoch $t$: 
\begin{subequations}
\begin{align}
    W^{t+1} &= W^t - \delta^t \nabla_W \tilde{\mathcal{L}}(W^t, \gamma^t),\\
    \gamma^{t+1} &= \gamma^t - \delta^t \left( \nabla_{\gamma}\tilde{\mathcal{L}}(W^t, \gamma^t) + \lambda \partial \|\gamma^t\|_1 \right),
\end{align} \label{eq:subgradient_step}
\end{subequations}
where $\tilde{\mathcal{L}}(W, \gamma) \coloneqq  \frac{1}{N} \sum_{i=1}^N \mathcal{L}(h(x_i,W, \gamma), y_i)$ and $\partial\|\cdot\|_1$
is the subgradient of the $\ell_1$ norm. 

By \eqref{eq:subgradient_step}, we observe that $\gamma$ is optimized by subgradient descent, which can lead to practical issues. When $\gamma_i = 0$ for some channel $i$, the subgradient needs to be chosen precisely. Not all subgradient vectors at a non-differentiable point decrease the value of \eqref{eq:network_slim} in each epoch \cite{beck2017first}, so we need to find one that does among the infinite number of choices.  In the numerical implementation of NS~\footnote{\url{https://github.com/Eric-mingjie/network-slimming}}, the subgradient $\zeta^t$ is selected such that $\zeta_i^t = 0$ by default when $\gamma_i^t = 0,$ but such selection is not verified to decrease the value of \eqref{eq:network_slim} in each epoch $t$. Lastly, subgradient descent only pushes the scaling factors of irrelevant channels to be near zero in value but not exactly zero. For this reason, when pruning a CNN, the user needs to determine the appropriate scaling factor threshold to remove its channels where no layers have zero channels and then fine tune it to restore its original accuracy. However, if too many channels are pruned that the fine-tuned accuracy is significantly less than the original, the user may waste time and resources by iterating the process of decreasing the threshold and fine tuning until the CNN attains acceptable accuracy and compression.

To develop an alternative algorithm that does not possess the practical issues of subgradient descent, we reformulate \eqref{eq:network_slim} as a constrained optimization problem by introducing an auxiliary variable $\xi$, giving us
\begin{equation}\label{eq:constrained_network_slim}
\begin{aligned}
\min_{W, \gamma, \xi} \quad  \tilde{\mathcal{L}}(W, \gamma) + \lambda \|\xi\|_1 \quad \textrm{s.t.} \quad  \xi = \gamma.
\end{aligned}
\end{equation}
However, we relax the constraint by a quadratic penalty with parameter $\beta > 0$, leading to a new unconstrained optimization problem:
\begin{align}\label{eq:new_network_slim}
    \min_{W, \gamma, \xi} \quad & \tilde{\mathcal{L}}(W, \gamma) + \lambda \|\xi\|_1 + \frac{\beta}{2} \|\gamma - \xi\|_2^2.
\end{align}
In \eqref{eq:network_slim}, the scaling factor vector $\gamma$ is optimized for both model accuracy and sparsity, which can be difficult to balance when training a CNN. However, in \eqref{eq:new_network_slim}, $\gamma$ is optimized for only model accuracy because it is a variable of the overall loss function $\tilde{\mathcal{L}}(W, \gamma)$ while $\xi$ is optimized only for sparsity because it is penalized by the $\ell_1$ norm. The quadratic penalty enforces $\gamma$ and $\xi$ to be similar in values, thereby ensuring $\gamma$ to be sparse.

Let $(W, \gamma)$ be a concatenated vector of $W$ and $\gamma$. We minimize \eqref{eq:new_network_slim} via alternating minimization, so for each epoch $t$, we solve the following subproblems:
\begin{subequations}
\begin{align} \label{eq:w_gamma_subprob}
    (W^{t+1}, \gamma^{t+1}) &\in \argmin_{W, \gamma} \tilde{\mathcal{L}}(W, \gamma) + \frac{\beta}{2} \|\gamma - \xi^t\|_2^2 \\ \label{eq:xi_subprob}
    \xi^{t+1} &\in \argmin_{\xi} \lambda \|\xi\|_1 + \frac{\beta}{2} \|\gamma^{t+1} - \xi\|_2^2.
\end{align}
\end{subequations}
Below, we describe how to solve each subproblem in details.
\subsubsection{$(W,\gamma)$-subproblem}
The $(W, \gamma)$-subproblem given in \eqref{eq:w_gamma_subprob} cannot be solved in closed form because the loss function $\tilde{\mathcal{L}}(W, \gamma)$ is a composition of several nonlinear functions. Typically, when training a CNN, this subproblem would be solved by (stochastic) gradient descent. To formulate \eqref{eq:w_gamma_subprob} as a gradient descent step, we follow a prox-linear strategy as follows:
\begin{align} \label{eq:w_gamma_min_prob}
\begin{split}
    &(W^{t+1}, \gamma^{t+1}) \in \argmin_{W, \gamma} \tilde{\mathcal{L}}(W^t, \gamma^t)+ \langle \nabla_W \tilde{\mathcal{L}}(W^t, \gamma^t), W - W^t\rangle\\ &+ \langle \nabla_{\gamma} \tilde{\mathcal{L}}(W^t, \gamma^t), \gamma - \gamma^t\rangle +\frac{\alpha}{2} \|W - W^{t}\|_2^2+\frac{\alpha}{2} \|\gamma - \gamma^{t}\|_2^2+ \frac{\beta}{2} \|\gamma - \xi^{t}\|_2^2,
    \end{split}
\end{align}
where $\alpha > 0$. By differentiating with respect to each variable, setting the partial derivative equal to zero, and solving for the variable, we have
\begin{subequations}
\begin{align} \label{eq:w_update}
    W^{t+1} &= W^t - \frac{1}{\alpha} \nabla_W \tilde{\mathcal{L}}(W^t, \gamma^t) \\ \label{eq:gamma_update}
    \gamma^{t+1} &= \frac{\alpha \gamma^t +\beta \xi^t}{\alpha+\beta} - \frac{1}{\alpha+\beta} \nabla_{\gamma} \tilde{\mathcal{L}}(W^t, \gamma^t).
\end{align}
\end{subequations}
We see that \eqref{eq:w_update} is gradient descent on $W^{t}$ with step size $\frac{1}{\alpha}$ while \eqref{eq:gamma_update} is gradient descent on a weighted average of $\gamma^t$ and $\xi^t$ with step size $\frac{1}{\alpha+\beta}$. These steps are straightforward to implement in practice when training a CNN because the gradient $(\nabla_W \tilde{\mathcal{L}}(W^t, \gamma^t),  \nabla_{\gamma} \tilde{\mathcal{L}}(W^t, \gamma^t))$ can be approximated by backpropagation. 
\subsubsection{$\xi$-subproblem}
To solve \eqref{eq:xi_subprob}, we perform a proximal update by minimizing the following subproblem:
\begin{align}\label{eq:xi_update}
    \xi^{t+1} &\in \argmin_{\xi} \lambda \|\xi\|_1  + \frac{\alpha}{2} \|\xi - \xi^{t}\|_2^2 + \frac{\beta}{2} \|\gamma^{t+1} - \xi\|_2^2.
\end{align}
Expanding it gives 
\begin{align*}
    \xi^{t+1}&=\argmin_{\xi} \|\xi\|_1 + \frac{1}{2 \left( \frac{\lambda}{\beta+\alpha} \right)} \left\| \xi - \frac{\alpha \xi^t + \beta \gamma^{t+1}}{\alpha + \beta} \right\|_2^2 =\mathcal{S}\left(\frac{\alpha \xi^t + \beta \gamma^{t+1}}{\alpha + \beta}, \frac{\lambda}{\beta+\alpha} \right),
\end{align*}
where $\mathcal{S}(x, \lambda)$ is the soft-thresholding operator defined by\\ $(\mathcal{S}(x, \lambda))_i  = \text{sign}(x_i) \max\{0, |x_i| - \lambda\}$
for each entry $i$. Therefore, $\xi$ is updated by performing soft thresholding on the weighted average between $\xi^t$ and $\gamma^{t+1}$. 

We summarize the new algorithm for NS in Algorithm \ref{alg:prox_network_slimming} as proximal NS.

 \begin{algorithm}[t!]
 \caption{Proximal NS: proximal algorithm for minimizing \eqref{eq:new_network_slim}}
 \scriptsize
 \label{alg:prox_network_slimming}
 \begin{algorithmic}[1]
 \renewcommand{\algorithmicrequire}{\textbf{Input:}}
 \renewcommand{\algorithmicensure}{\textbf{Output:}}
 \REQUIRE Regularization parameter $\lambda$, proximal parameter $\alpha$, penalty parameter $\beta$
 \\ Initialize $W^1, \xi^1$ with random values.\\
 Initialize $\gamma^1$ such that $\gamma_i = 0.5$ for each channel $i$. 
  \FOR {each epoch $t=1, \ldots, T$}
   \STATE $W^{t+1} = W^t -\frac{1}{\alpha} \nabla_W \tilde{\mathcal{L}}(W^t, \gamma^t)$  by stochastic gradient descent or variant.
   \STATE $\gamma^{t+1} = \frac{\alpha \gamma^t +\beta \xi^t}{\alpha+\beta} - \frac{1}{\alpha+\beta} \nabla_{\gamma} \tilde{\mathcal{L}}(W^t, \gamma^t)$ by stochastic gradient descent or variant.
   \STATE $\xi^{t+1} = \mathcal{S}\left(\frac{\alpha \xi^t + \beta \gamma^{t+1}}{\alpha + \beta}, \frac{\lambda}{\beta+\alpha} \right).$
  \ENDFOR
 \end{algorithmic} 
 \end{algorithm}
\section{Convergence Analysis}
To establish global convergence of proximal NS, we present relevant definitions and assumptions.

\begin{definition}[\cite{bolte2014proximal}]
A proper, lower-semicontinuous function $f: \mathbb{R}^m \rightarrow (-\infty, \infty]$ satisfies the Kurdyka-Łojasiewicz (KL) property at a point $\bar{x} \in \text{dom}(\partial f) \coloneqq \{x \in \mathbb{R}^m: \partial{f}(x) \neq \varnothing\}$ if there exist $\eta \in (0, +\infty]$, a neighborhood $U$ of $\bar{x}$, and a continuous concave function $\phi:[0, \eta) \rightarrow [0, \infty)$ with the following properties: (i) $\phi(0) = 0$; (ii) $\phi$ is continuously differentiable on $(0, \eta)$; (iii) $\phi'(x) > 0$ for all $x \in (0, \eta)$; and (iv) for any $x \in U$ with $f(\bar{x}) < f(x) < f(\bar{x})+\eta$, it holds that $\phi'(f(x) - f(\bar{x}))\text{dist}(0, \partial f (x)) \geq 1.$
If $f$ satisfies the KL property at every point $x \in \text{dom}(\partial f)$, then $f$ is called a KL function.  
\end{definition}
\begin{assumption} \label{assume:loss_function}
Suppose that
\begin{enumerate}[label=\alph*)]
    \item $\tilde{\mathcal{L}}(W, \gamma)$ is a proper, differentiable, and nonnegative function.
    \item $\nabla \tilde{\mathcal{L}}(W, \gamma)$ is Lipschitz continuous with constant $L$. 
    \item $\tilde{\mathcal{L}}(W, \gamma)$ is a KL function. 
\end{enumerate}
\end{assumption}
\begin{remark} Assumption \ref{assume:loss_function} (a)-(b) are common in nonconvex analysis (e.g., \cite{bolte2014proximal}). For Assumption \ref{assume:loss_function}, most commonly used loss functions for CNNs are verified to be KL functions \cite{zeng2019global}. Some CNN architectures do not satisfy Assumption \ref{assume:loss_function}(a) when they contain nonsmooth functions and operations, such as the ReLU activation functions and max poolings.  However, these functions and operations can be replaced with their smooth approximations. For example, the smooth approximation of ReLU is the softplus function $\frac{1}{c} \log(1+\exp(c x))$ for some parameter $c>0$ while the smooth approximation of the max function for max pooling is the softmax function $\sum_{i=1}^n \frac{x_i e^{c x_i}}{\sum_{i=1}^n e^{cx_i}}$ for some parameter $c >0$. Besides, Fu \etal \cite{fu2022exploring} made a similar assumption to establish convergence for their algorithm designed for weight and filter pruning.  Regardless, our numerical experiments demonstrate that our proposed algorithm still converges for CNNs containing ReLU activation functions and max pooling.
\end{remark}
For brevity, we denote
\begin{align*}
F(W, \gamma, \xi) \coloneqq \tilde{\mathcal{L}}(W, \gamma) + \lambda \|\xi\|_1 + \frac{\beta}{2} \|\gamma -\xi\|_2^2.
\end{align*}
Now, we are ready to present the main theorem:
\begin{theorem}\label{thm:global_conv}
Under Assumption \ref{assume:loss_function}, if $\{(W^t, \gamma^t, \xi^t)\}_{t=1}^{\infty}$ generated by Algorithm \ref{alg:prox_network_slimming} is bounded and we have $\alpha > L$, then $\{(W^t, \gamma^t, \xi^t)\}_{t=1}^{\infty}$ converges to a critical point $(W^*, \gamma^*, \xi^*)$ of $F$. 
\end{theorem}
The proof is delayed to the appendix. It requires satisfying the sufficient decrease property in $F$ and the relative error property of $\partial{F}$  \cite{bolte2014proximal}.
\section{Numerical Experiments}
We evaluate proximal NS on VGG-19 \cite{simonyan2014very}, DenseNet-40 \cite{huang2017densely,huang2019convolutional}, and ResNet-110/164 \cite{he2016deep} trained on CIFAR 10/100 \cite{krizhevsky2009learning}. The CIFAR 10/100 dataset \cite{krizhevsky2009learning} consists of 60,000 natural images of resolution $32 \times 32$ with 10/100 categories. The dataset is split into two sets: 50,000 training images and 10,000 test images. As done in recent works \cite{he2016deep,liu2017learning}, standard augmentation techniques (e.g., shifting, mirroring, and normalization) are applied to the images before training and testing. The code for proximal NS is available at \url{https://github.com/kbui1993/Official-Proximal-Network-Slimming}.

\subsection{Implementation Details}
For CIFAR 10/100, the implementation is mostly the same as in \cite{liu2017learning}. Specifically, we train the networks from scratch for 160 epochs using stochastic gradient descent with initial learning rate at 0.1 that reduces by a factor of 10 at the 80th and 120th epochs. Moreover, the models are trained with weight decay $10^{-4}$ and Nesterov momentum of 0.9 without damping. The training batch size is 64. However, the parameter $\lambda$ is set differently. In our numerical experiments, using Algorithm \ref{alg:prox_network_slimming}, we set $\xi \sim \text{Unif}[0.47,0.50]$ for all networks while $\lambda =0.0045$ and $\beta = 100$ for VGG-19, $\lambda = 0.004$ and $\beta = 100$ for DenseNet-40, and $\lambda = 0.002$ and $\beta = 1.0, 0.25$ for ResNet-110 and ResNet-164, respectively. We have initially $\alpha = 10$, the reciprocal of the learning rate, and it changes accordingly to the learning rate schedule. A model is trained  five times on NVIDIA GeForce RTX 2080 for each network and dataset to obtain the average statistics.

 \begin{table}[t!]
\caption{The average number of scaling factors equal to zero at the end of training. Each architecture is trained five times per dataset. }
\label{tab:zero_scaling_factor}
    \centering
    \begin{tabular}{|l||c||c||c|}
    \hline
     & & \multicolumn{1}{c||}{CIFAR 10} & \multicolumn{1}{c|}{CIFAR 100}  \\
        Architecture & \makecell{Total Channels/$\gamma_i$} & \makecell{Avg. Number\\ of $\gamma_i = 0$} &  \makecell{Avg. Number\\ of $\gamma_i = 0$}\\ \hline
        VGG-19 & 5504 & 4105.2 &  3057.0\\ \hline
        DenseNet-40 & 9360 & 6936.4 & 6071.6\\ \hline
        ResNet-164 & 12112 & 8765.4 &  7115.8  \\ \hline
    \end{tabular}
    \vspace{-5mm}
\end{table} 

\subsection{Results}
We apply proximal NS to train VGG-19, DenseNet-40, and ResNet-164 on CIFAR 10/100. According to Table \ref{tab:zero_scaling_factor}, proximal NS drives a significant number of scaling factors to be exactly zeroes for each trained CNN. In particular, for VGG-19 and DenseNet-40, at least 55\% of the scaling factors are zeroes while for ResNet-164, at least 58\% are zeroes. We can safely remove the channels with zero scaling factors because they are unnecessary for inference. Unlike the original NS~\cite{liu2017learning}, proximal NS does not require us to select a scaling factor threshold based on how many channels to remove and how much accuracy to sacrifice. 

\begin{table*}[t!!!!]
 \scriptsize
\caption{Results between the different NS methods on CIFAR 10/100. Average statistics are obtained by training the baseline architectures and original NS five times, while the results for variational NS are originally reported from \cite{zhao2019variational}.}
\begin{subtable}{\textwidth}
\caption{CIFAR 10}
\label{tab:cifar10_tab}
    \begin{tabular}{|l|c|c|c|c|c|c|}
    \hline
        Architecture & Method & \makecell{Avg. Training Time\\ per Epoch (s)\\
        Pre-Pruned/Fine Tuned} &\makecell{\%\\ Channels\\ Pruned} & \makecell{\%\\ Param.\\ Pruned} & \makecell{\%\\ FLOPS\\ Pruned} &\makecell{Test Accuracy (\%) \\
       Post Pruned/Fine Tuned} \\ \hline
        \multirow{3}{*}{VGG-19} & Baseline & 38.10/----&N/A & N/A & N/A & 93.83/---- \\ 
        ~ & Original NS \cite{liu2017learning} & 40.39/29.40 &74.00* & 90.22 & 54.67 & 10.00/93.81\\ 
        ~ & Proximal NS (ours) &42.71/30.39& 74.59 & \textbf{91.17} & \textbf{57.54} & 93.71/93.38 \\ \hline
        \multirow{3}{*}{DenseNet-40} & Baseline & 117.45/----& N/A & N/A & N/A & 94.25/---- \\ 
        ~ & Original NS \cite{liu2017learning} & 119.49/74.45 &74.01 & 67.13 & \textbf{60.46} & 41.46/93.94 \\
        ~ & VCP \cite{zhao2019variational} & Not Reported & 60.00 & 59.67 & 44.78 & 93.16/---- \\ 
         ~ & Proximal NS (ours) & 118.86/76.10 &74.11& \textbf{67.75}  & 57.35 & 93.58/93.64 \\\hline
        \multirow{3}{*}{ResNet-164} & Baseline & 146.41/---- &N/A & N/A & N/A & 94.75/---- \\ 
        ~ & Original NS \cite{liu2017learning} & 151.62/112.80&71.98& 52.95 & 59.27 & 16.61/93.21 \\ 
        ~ & VCP \cite{zhao2019variational} & Not Reported & 74.00 & 56.70 & 49.08 & 93.16/---- \\ 
        ~ & Proximal NS (ours) & 150.13/114.26 &72.37 & \textbf{65.84}& \textbf{63.54} & 93.19/93.41 \\ \hline
    \end{tabular}
    \end{subtable}\\
      \begin{subtable}
      {\textwidth}\caption{CIFAR 100}\label{tab:cifar100_tab}
\scriptsize
\begin{tabular}{|l|c|c|c|c|c|c|}
    \hline
        Architecture & Method & \makecell{Avg. Training Time\\ per Epoch (s)\\Pre-Pruned/Fine Tuned
        }  &\makecell{\%\\ Channels\\ Pruned} & \makecell{\%\\ Param.\\ Pruned} & \makecell{\%\\ FLOPS\\ Pruned} & \makecell{Test Accuracy (\%) \\
       Post Pruned/Fine Tuned} \\ \hline
        \multirow{3}{*}{VGG-19} & Baseline & 37.83& N/A & N/A & N/A & 72.73/---- \\ 
        ~ & Original NS \cite{liu2017learning} & 39.98/30.74 &55.00  & 78.53& 38.66 & 1.00/72.91 \\ 
        ~ & Proximal NS (ours) & 42.31/30.04 &55.54 & \textbf{79.62} & \textbf{41.17} & 72.81/72.70 \\ \hline
        \multirow{3}{*}{DenseNet-40} & Baseline & 117.17 &N/A & N/A & N/A & 74.55/---- \\ 
         ~ & Original NS \cite{liu2017learning} & 119.32/77.95 &65.01 & \textbf{59.29} & \textbf{52.61} & 25.96/74.50 \\ 
        ~ & VCP \cite{zhao2019variational} & Not Reported & 37.00 & 37.73 & 22.67 & 72.19/---- \\ 
          ~ & Proximal NS (ours) & 120.89/82.92  & 64.87& 59.15 & 45.00 & 73.70/73.98\\ \hline
        \multirow{3}{*}{ResNet-164} & Baseline & 145.37 &N/A & N/A & N/A & 76.79/---- \\ 
        ~ & Original NS \cite{liu2017learning} & 150.65/115.95& 59.00 & 26.66 & 45.17 & 2.39/76.68  \\
        ~ &VCP \cite{zhao2019variational} & Not Reported& 47.00 & 17.59 & 27.16 & 73.76/---- \\ 
        ~ & Proximal NS (ours) & 149.15/117.88 & 58.75 & \textbf{42.28} & \textbf{47.93} & 75.26/75.68\\ \hline
    \end{tabular}
        \scriptsize{\\$^*$ This is the maximum possible for all five networks to remain functional for inference.}
    \vspace{-5mm}
    \end{subtable}
\end{table*}
We compare proximal NS with the original NS \cite{liu2017learning} and variational CNN pruning (VCP) \cite{zhao2019variational}, a Bayesian version of NS. To evaluate the effect of regularization and pruning on accuracy, we include the baseline accuracy, where the architecture is trained without any regularization on the scaling factors. For completeness, the models trained with original NS and proximal NS are fine tuned with the same setting as the first time training but without $\ell_1$ regularization on the scaling factors. The results are reported in Tables \ref{tab:cifar10_tab}-\ref{tab:cifar100_tab}. 

After the first round of training, proximal NS outperforms both the original NS and VCP in test accuracy while reducing a significant amount of parameters and FLOPs. Because proximal NS trains a model towards a sparse structure, the model accuracy is less than the baseline accuracy by at most 1.56\% and it remains the same between before and after pruning, a property that the original NS does not have. Although VCP is designed to preserve test accuracy after pruning, it does not compress as well as proximal NS for all architectures. With about the same proportion of channels pruned as the original NS, proximal NS saves more FLOPs for both VGG-19 and ResNet-164 and generally more parameters for all networks.

To potentially improve test accuracy, the pruned models from the original and proximal NS are fine tuned. For proximal NS, test accuracy of the pruned models improve slightly by at most 0.42\% for DenseNet-40 and ResNet-164 while worsen for VGGNet-19. Moreover, proximal NS is outperformed by the original NS in fine-tuned test accuracy for all models trained on CIFAR 100.

A more accurate model from original NS might be preferable. However, the additional fine tuning step requires a few more training hours to obtain an accuracy that is up to 1.5\% higher than the accuracy of a pruned model trained once by proximal NS. For example, for ResNet-164 trained on CIFAR 100, proximal NS takes about 7 hours to attain an average accuracy of 75.26\% while the original NS requires about 12 hours to achieve 1.42\% higher accuracy. Therefore, the amount of time and resources spent training for an incremental improvement may not be worthwhile. 

Finally, we compare proximal NS with other pruning methods applied to Densenet-40 and ResNet-110 trained on CIFAR 10. The other pruning methods, which may require fine tuning, are L1 \cite{li2016pruning}, GAL \cite{lin2019towards}, and Hrank \cite{lin2020hrank}.  For DenseNet-40, proximal NS prunes the most parameters and the second most FLOPs while having comparable accuracy as the fine-tuned Hrank and post-pruned GAL-0.05. For ResNet-110, proximal NS has better compression than L1, GAL-0.5, and Hrank with its post-pruned accuracy better than GAL-0.5's fine-tuned accuracy and similar to L1's fine-tuned accuracy. Although GAL or Hrank might be advantageous to use to obtain a sparse, accurate CNN, they have additional requirements besides fine tuning. GAL \cite{lin2019towards} requires an accurate baseline model available for knowledge distillation. For Hrank \cite{lin2020hrank}, the compression ratio needs to be specified for each convolutional layer, thereby making hyperparameter tuning more complicated. 

Overall, proximal NS is a straightforward algorithm that yields a generally more compressed and accurate model than the other methods in one training round. Although its test accuracy after one round is slightly lower than the baseline accuracy, it is expected because of the sparsity--accuracy trade-off and being a prune-while-training algorithm (which automatically identifies the insignificant channels during training) as discussed in \cite{shen2022prune}. Lastly, the experiments show that fine tuning the compressed models trained by proximal NS marginally improves the test accuracy, which makes fine tuning wasteful.

\begin{table}[!t]
    \centering
    \caption{Comparison of Proximal NS with other pruning methods on CIFAR 10.}
    \label{tab:compare}
    \scriptsize
    \begin{tabular}{|c|c|c|c|}
    \hline
       Architecture & Method & \makecell{\% Param./FLOPs\\ Pruned} & \makecell{Test Accuracy (\%) \\
       Post Pruned/Fine Tuned}  \\ \hline
        \multirow{3}{*}{DenseNet-40} & Hrank \cite{lin2020hrank} & 53.80/61.00 &  ----/93.68 \\ 
        ~ & GAL-0.05 \cite{lin2019towards} & 56.70/54.70 & 93.53/94.50 \\
        ~ & Proximal NS (Ours) & 67.75/57.54 & 93.58/93.64 \\ \hline
        \multirow{4}{*}{ResNet-110} & L1 \cite{li2016pruning} & 32.60/38.70 &  ----/93.30 \\ 
        ~ & GAL-0.5\cite{lin2019towards} & 44.80/48.50 & 92.55/92.74 \\ 
        ~ & Hrank \cite{lin2020hrank} & 39.40/41.20 & ----/94.23 \\ 
        ~ & Proximal NS (Ours) & 50.70/48.54 & 93.25/93.27\\
        \hline
    \end{tabular}
    \vspace{-5mm}
\end{table}

\section{Conclusion}
We develop a channel pruning algorithm called proximal NS with global convergence guarantee. It trains a CNN towards a sparse, accurate structure, making fine tuning optional. In our experiments, proximal NS can effectively compress CNNs with accuracy slightly less than the baseline. Because fine tuning CNNs trained by proximal NS marginally improves test accuracy, we will investigate modifying the algorithm to attain significantly better fine-tuned accuracy. 

For future direction, we shall study proximal cooperative neural architecture search \cite{rarts2022,coop2021}
and include nonconvex, sparse regularizers, such as  $\ell_1 - \ell_2$ \cite{yin2015minimization} and transformed $\ell_1$ \cite{zhang2018minimization}.  
\appendix
\section{Appendix}
First, we introduce important definitions and lemmas from variational analysis.

\begin{definition}[\cite{rockafellar2009variational}]
Let $f:\mathbb{R}^{n} \rightarrow (-\infty, +\infty]$ be a proper and lower semicontinuouous function. 
\begin{enumerate}[label=(\alph*)]
    \item The Fr\'echet subdifferential of $f$ at the point $x \in \text{dom } f \coloneqq \{x \in \mathbb{R}^n: f(x) < \infty\}$ is the  set
    \begin{align*}
        \hat{\partial}f(x) = \left\{ v \in \mathbb{R}^{n^2}: \liminf_{y \neq x, y \rightarrow x} \frac{f(y)-f(x) - \langle v, y-x \rangle}{\|y-x\|} \geq 0 \right\}.
    \end{align*}
\item The limiting subdifferential of $f$ at the point $x \in \text{dom } f$ is the set
\begin{align*}
    \partial f(x) = \left\{ v \in \mathbb{R}^{n^2}: \exists \{(x^t,y^t)\}_{t=1}^{\infty} \text{ s.t. } x^t \rightarrow x, f(x^t) \rightarrow f(x), \hat{\partial}f(x^t) \ni y^t \rightarrow y \right\}.
\end{align*}
\end{enumerate}
\end{definition}
\begin{lemma}[Strong Convexity Lemma \cite{beck2017first}]\label{lemma:convex}
 A function $f(x)$ is called strongly convex with parameter $\mu$ if and only if one of the following conditions holds:
 \begin{enumerate}[label=\alph*)]
     \item $g(x) = f(x) - \frac{\mu}{2} \|x\|_2^2$ is convex.
     \item $f(y) \geq f(x) + \langle \nabla f(x), y-x \rangle + \frac{\mu}{2} \|y-x\|_2^2, \; \forall x,y$.
 \end{enumerate}
 \end{lemma}
 \begin{lemma}[Descent Lemma \cite{beck2017first}]\label{lemma:descent}
 If $\nabla f(x)$ is Lipschitz continuous with parameter $L >0$, then 
 \begin{align*}
 f(y) \leq f(x) + \langle \nabla f(x), y- x \rangle + \frac{L}{2} \|x-y\|_2^2,\; \forall x,y.
 \end{align*}
 \end{lemma}

For brevity, denote $\tilde{W} \coloneqq (W, \gamma)$, the overall set of weights in a CNN, and $Z \coloneqq (\tilde{W}, \xi) = (W, \gamma, \xi)$. Before proving Theorem \ref{thm:global_conv}, we prove some necessary lemmas.
\begin{lemma}[Sufficient Decrease]\label{lemma:suff_decrease}
Let $\{Z^t\}_{t=1}^{\infty}$ be a sequence generated by Algorithm \ref{alg:prox_network_slimming}. Under Assumption \ref{assume:loss_function}, we have
\begin{align} 
F(Z^{t+1}) - F(Z^t) \leq \frac{L- \alpha}{2} \|Z^{t+1} - Z^t\|_2^2.
\end{align}
for all $t \in \mathbb{N}$. In addition, when $\alpha >L$, we have
\begin{align} \label{eq:finite_length}
     \sum_{t=1}^{\infty} \|Z^{t+1} - Z^t\|_2^2 < \infty.
\end{align} 
\end{lemma}
\begin{proof}
First we define the function
\begin{gather}
\begin{aligned}
    L_{t}(\tilde{W}) &= \tilde{\mathcal{L}}(\tilde{W}^t) + \langle \nabla \tilde{\mathcal{L}}(\tilde{W}^{t}), \tilde{W} - \tilde{W}^t \rangle+ \frac{\alpha}{2} \|\tilde{W}- \tilde{W}^t \|_2^2 + \frac{\beta}{2} \|\gamma - \xi^t \|_2^2. 
\end{aligned}
\end{gather}
We observe that $L_t$ is strongly convex with respect to $\tilde{W}$ with parameter $\alpha$. Because $\nabla L_t(\tilde{W}^{t+1}) = 0$ by \eqref{eq:w_gamma_min_prob}, we use Lemma \ref{lemma:convex} to obtain 
\begin{gather}
\begin{aligned}
    L_t(\tilde{W}^t) &\geq L_t(\tilde{W}^{t+1}) + \langle \nabla L_t(\tilde{W}^{t+1}), \tilde{W}^t - \tilde{W}^{t+1}\rangle + \frac{\alpha}{2} \|\tilde{W}^{t+1} - \tilde{W}^t \|_2^2\\
    & \geq L_t(\tilde{W}^{t+1}) + \frac{\alpha}{2} \|\tilde{W}^{t+1} -\tilde{W}^t\|_2^2,
\end{aligned}
\end{gather}
which simplifies to
\begin{gather}
\begin{aligned}
    \tilde{\mathcal{L}}(\tilde{W}^t) + \frac{\beta}{2} \|\gamma^t - \xi^t\|_2^2  - \alpha  \|\tilde{W}^{t+1} - \tilde{W}^t\|_2^2 \geq &\tilde{\mathcal{L}}(\tilde{W}^t) + \langle \nabla \tilde{\mathcal{L}}(\tilde{W}^t), \tilde{W}^{t+1} - \tilde{W}^t \rangle\\ &+ \frac{\beta}{2} \|\gamma^{t+1} - \xi^t\|_2^2.
\end{aligned}
\end{gather}
Since $\nabla \tilde{\mathcal{L}}(\tilde{W})$ is Lipschitz continuous with constant $L$, we have
\begin{align}
    \tilde{\mathcal{L}}(\tilde{W}^{t+1}) \leq \tilde{\mathcal{L}}(\tilde{W}^t) + \langle \nabla  \tilde{\mathcal{L}}(\tilde{W}^{t}), \tilde{W}^{t+1} - \tilde{W}^{t} \rangle + \frac{L}{2} \|\tilde{W}^{t+1} - \tilde{W}^{t} \|_2^2
\end{align}
by Lemma \ref{lemma:descent}. Combining the previous two inequalities gives us
\begin{align*}
    \tilde{\mathcal{L}}(\tilde{W}^t) + \frac{\beta}{2} \|\gamma^t - \xi^t\|_2^2  + \frac{L - 2\alpha}{2}  \|\tilde{W}^{t+1} - \tilde{W}^t\|_2^2 \geq \tilde{\mathcal{L}}(\tilde{W}^{t+1}) + \frac{\beta}{2} \|\gamma^{t+1} - \xi^t\|_2^2.
\end{align*}
Adding the term $\lambda \|\xi^t\|_1$ on both sides and rearranging the inequality give us 
\begin{align}\label{eq:first_sufficient_descent}
     F(\tilde{W}^{t+1}, \xi^t) - F(Z^t) &\leq \frac{L-2\alpha}{2} \|\tilde{W}^{t+1} - \tilde{W}^{t} \|_2^2
\end{align}

By \eqref{eq:xi_update}, we have
\begin{align*}
    \lambda \|\xi^{t+1} \|_1 + \frac{\beta}{2} \|\gamma^{t+1} - \xi^{t+1}\|_2^2 + \frac{\alpha}{2} \|\xi^{t+1} - \xi^t\|_2^2 \leq \lambda \|\xi^t\|_1 + \frac{\beta}{2} \|\gamma^{t+1} - \xi^t\|_2^2.
\end{align*}
Adding $\tilde{\mathcal{L}}(\tilde{W}^{t+1})$ on both sides and rearranging the inequality give
\begin{align}\label{eq:second_sufficient_decrease}
     F(Z^{t+1}) - F(\tilde{W}^{t+1}, \xi^t) &\leq - \frac{\alpha}{2} \|\xi^{t+1} - \xi^t\|_2^2
\end{align}
Summing up \eqref{eq:first_sufficient_descent} and \eqref{eq:second_sufficient_decrease} and rearranging them, we have
\begin{align} F(Z^{t+1}) - F(Z^t) \leq \frac{L-2\alpha}{2} \|\tilde{W}^{t+1}- \tilde{W}^{t} \|_2^2  - \frac{\alpha}{2} \|\xi^{t+1} - \xi^t\|_2^2 \leq \frac{L-\alpha}{2} \|Z^{t+1} - Z^t\|_2^2. 
\end{align}
Summing up the inequality for $t = 1, \ldots, N-1$, we have
\begin{align*}
    \sum_{t=1}^{N-1} \frac{\alpha-L}{2} \|Z^{t+1} - Z^t\|_2^2 \leq F(Z^1) - F(Z^N) \leq F(Z^1).
\end{align*}
Because $\alpha >L$, the left-hand side is nonnegative, so as $N \rightarrow \infty$, we have \eqref{eq:finite_length}.
\end{proof}
\begin{lemma}[Relative error property] \label{lemma:relative_error_prop}
Let $\{Z^t\}_{t=1}^{\infty}$ be a sequence generated by Algorithm \ref{alg:prox_network_slimming}. Under Assumption \ref{assume:loss_function}, for any $t \in \mathbb{N}$, there exists some $w^{t+1} \in \partial F(Z^{t+1})$ such that
\begin{align}
    \|w^{t+1}\|_2 \leq (3\alpha + 2L +\beta) \left\| Z^{t+1} -  Z^t\right\|_2.
\end{align}
\end{lemma}
\begin{proof}
We note that
\begin{subequations}
\begin{align} \label{eq:subgrad_w_f}
\nabla_W \tilde{\mathcal{L}}(\tilde{W}^{t+1})  &\in \partial_W F(Z^{t+1}), \\\label{eq:subgrad_gamma_f}
\nabla_{\gamma}\tilde{\mathcal{L}}(\tilde{W}^{t+1}) + \beta(\gamma^{t+1} - \xi^{t+1}) &\in  \partial_{\gamma} F(Z^{t+1}),\\\label{eq:subgrad_xi_f}
\lambda \partial_{\xi}\|\xi^{t+1}\|_1 -\beta(\gamma^{t+1} - \xi^{t+1}) &\in \partial_{\xi}F(Z^{t+1}).
\end{align}
\end{subequations}
By the first-order optimality conditions of \eqref{eq:w_gamma_min_prob} and \eqref{eq:xi_update}, we obtain
\begin{subequations}
\begin{align} \label{eq:w_first_order}
    \nabla_W \mathcal{\tilde{L}}(\tilde{W}^{t}) + \alpha(W^{t+1}-W^t) &= 0,\\ \label{eq:gamma_first_order}
    \nabla_{\gamma} \mathcal{\tilde{L}}(\tilde{W}^{t}) + \alpha(\gamma^{t+1} - \gamma^t) + \beta(\gamma^{t+1} - \xi^t) &= 0, \\ \label{eq:xi_first_order}
    \lambda \partial_{\xi}\|\xi^{t+1}\|_1 +\alpha(\xi^{t+1} - \xi^t) - \beta(\gamma^{t+1} - \xi^{t+1}) &\ni 0. 
\end{align}
\end{subequations}
Combining \eqref{eq:subgrad_w_f} and \eqref{eq:w_first_order}, \eqref{eq:subgrad_gamma_f} and \eqref{eq:gamma_first_order}, and \eqref{eq:subgrad_xi_f} and \eqref{eq:xi_first_order}, we obtain
\begin{subequations}
\begin{align}
    &\nabla_{W}\tilde{\mathcal{L}}(\tilde{W}^{t+1}) -  \nabla_W \mathcal{\tilde{L}}(\tilde{W}^{t}) - \alpha (W^{t+1} - W^t) =w_1^{t+1} \in \partial_{W} F(Z^{t+1}), \\
    &\nabla_{\gamma}\tilde{\mathcal{L}}(\tilde{W}^{t+1}) -  \nabla_{\gamma} \mathcal{\tilde{L}}(\tilde{W}^{t}) - \alpha (\gamma^{t+1} - \gamma^t) - \beta(\xi^{t+1} - \xi^{t}) =w_2^{t+1} \in \partial_{\gamma} F(Z^{t+1}),\\
    &-\alpha(\xi^{t+1} - \xi^t) =w_3^{t+1}  \in \partial_{\xi}F(Z^{t+1}),
\end{align}
\end{subequations}
where  $w^{t+1}=(w_1^{t+1}, w_2^{t+1}, w_3^{t+1}) \in \partial F(Z^{t+1})$. As a result, by triangle inequality and Lipschitz continuity of $\nabla \tilde{\mathcal{L}}$, we have
\begin{align*}
    &\|w_1^{t+1} \|_2 \leq \alpha \|W^{t+1} - W^t\|_2 + \| \nabla_{W}\tilde{\mathcal{L}}(\tilde{W}^{t+1}) -  \nabla_W \mathcal{\tilde{L}}(\tilde{W}^{t})\|_2\\ &\leq \alpha \|W^{t+1} - W^t\| + L \|\tilde{W}^{t+1} - \tilde{W}^{t}\|_2 
     \leq (\alpha + L) \|Z^{t+1} - Z^t\|_2,
\end{align*}
\begin{align*}
    \|w_2^{t+1} \|_2 &\leq \alpha\|\gamma^{t+1} - \gamma^t\|_2 + \beta \|\xi^{t+1} - \xi^{t}\|_2 + \|\nabla_{\gamma}\tilde{\mathcal{L}}(\tilde{W}^{t+1}) -  \nabla_{\gamma} \mathcal{\tilde{L}}(\tilde{W}^{t})\|_2 \\
    &\leq (\alpha +L)  \|\tilde{W}^{t+1} - \tilde{W}^{t}\|_2 +  \beta \|\xi^{t+1} - \xi^{t}\|_2\leq (\alpha + \beta+ L)  \|Z^{t+1} - Z^t\|_2 , 
\end{align*}
and
\begin{align*}
    \|w_3^{t+1}\|_2 \leq \alpha \|\xi^{t+1} - \xi^t\|_2 \leq \alpha  \|Z^{t+1}- Z^t\|_2 .
\end{align*}
Therefore, for all $t \in \mathbb{N}$, we have
\begin{align*}
    \|w^{t+1}\|_2 \leq \|w_1^{t+1}\|_2+ \|w_2^{t+1}\|_2+ \|w_3^{t+1}\|_2 \leq (3\alpha + 2L +\beta) \left\| Z^{t+1} -  Z^t\right\|_2.
\end{align*}
\end{proof}

\begin{proof}[Proof of Theorem \ref{thm:global_conv}]
The result follows from Lemmas  \ref{lemma:suff_decrease}-\ref{lemma:relative_error_prop} combined with \cite[Theorem 1]{bolte2014proximal}
\end{proof}
%
%
\bibliographystyle{splncs04}
\bibliography{egbib}
%




\end{document}